\documentclass{article}

\usepackage{microtype}
\usepackage{graphicx}
\usepackage{subcaption}
\usepackage{booktabs}
\usepackage{hyperref}

\usepackage[accepted]{icml2026}

\usepackage{amsmath}
\usepackage{amssymb}
\usepackage{mathtools}
\mathtoolsset{showonlyrefs=false}
\usepackage{amsthm}
\usepackage{algorithm}
\usepackage{algorithmic}

\usepackage[capitalize,noabbrev]{cleveref}
\usepackage{tikz}

\theoremstyle{plain}
\newtheorem{theorem}{Theorem}[section]
\newtheorem{proposition}[theorem]{Proposition}
\newtheorem{lemma}[theorem]{Lemma}
\newtheorem{corollary}[theorem]{Corollary}
\theoremstyle{definition}
\newtheorem{definition}[theorem]{Definition}

\theoremstyle{remark}
\newtheorem{remark}[theorem]{Remark}

\usepackage[textsize=tiny]{todonotes}

\definecolor{darkblue}{rgb}{0,0,.4}
\definecolor{darkgreen}{rgb}{0,.4,0}

\renewcommand{\P}{\mathbb{P}}

\crefname{algorithm}{Algorithm}{Algorithms}
\Crefname{algorithm}{Algorithm}{Algorithms}

\providecommand{\citet}[1]{\cite{#1}}

\newcommand{\R}{\mathbb{R}}
\newcommand{\E}{\mathbb{E}}
\newcommand{\x}{\mathbf{x}}

\newcommand{\Pareto}{\text{Pareto}}
\newcommand{\Exp}{\text{Exp}}

\icmltitlerunning{Tail Annealing for Heavy-Tailed Flow Matching}

\begin{document}

\twocolumn[
  \icmltitle{Tail Annealing for Heavy-Tailed Flow Matching}

  \begin{icmlauthorlist}
    \icmlauthor{Jean Pachebat}{cmap}
  \end{icmlauthorlist}

  \icmlaffiliation{cmap}{CMAP, \'Ecole Polytechnique, Institut Polytechnique de Paris, France}
  \icmlcorrespondingauthor{Jean Pachebat}{jean.pachebat@polytechnique.edu}

  \icmlkeywords{Score-Based Models, Heavy-Tailed Distributions, Generative Models, Exponential Family, Stochastic Interpolants}

  \vskip 0.3in
]

\printAffiliationsAndNotice{}


\begin{abstract}
Standard generative models struggle with heavy-tailed data: Lipschitz architectures cannot produce power-law tails from Gaussian noise, and interpolating between heavy-tailed data and Gaussians is ill-posed. We propose a simple fix: apply the soft-log transform $\phi(x) = \mathrm{sign}(x) \cdot \log(1 + |x|)$ coordinate-wise to data before training, then exponentiate samples after generation. A Hill diagnostic decides per-coordinate whether to transform, leaving light-tailed margins untouched at no added complexity. This compresses heavy tails into a range where standard flow matching succeeds, without heavy-tailed base distributions or architectural modifications. We provide theoretical intuition for why this works: the log-transform maps Pareto tails to exponentials, and the induced dynamics implement a form of tail annealing via power transformations. On a 144-configuration multivariate benchmark (3 copulas, $d$ up to 100, 4 tail indices), Log-FM dominates specialized baselines on $W_1$, CVaR$_{99}$, and extreme-quantile metrics, and is the only method with zero severe divergences across 2{,}880 runs.
\end{abstract}



\section{Introduction}

Heavy-tailed distributions are ubiquitous in real-world phenomena: financial returns, insurance claims, earthquake magnitudes, network degrees, and climate extremes. In these domains, rare events dominate outcomes: asset returns exhibit empirically documented fat tails that Gaussian models systematically underestimate \citep{cont2001}, and similar departures from Gaussianity are well-known in insurance \citep{embrechts2013modelling} and natural-hazard data. Generating realistic synthetic data for stress testing, simulation, or data augmentation requires generative models that capture tail behavior.

Yet standard generative models fail on heavy-tailed data, both in theory and in practice. \citet{jaini2020tails} proved that Lipschitz transformations preserve tail type: light-tailed inputs yield light-tailed outputs. Since most practical architectures (MLPs in GANs, coupling layers in normalizing flows, and diffusion denoisers) use Lipschitz-continuous functions for training stability, they cannot generate heavy tails from Gaussian noise without explicit tail-modifying components.

\paragraph{Prior approaches.} Existing methods address heavy tails through two strategies. \citet{jaini2020tails} replace Gaussian noise with heavy-tailed base distributions (Student-$t$), but this introduces training instability. \citet{hickling2025tail} take a different approach: generate samples from a standard flow, then apply a tail-modifying transformation to the \emph{output}. Their method requires estimating the tail parameter $\lambda = 1/\nu$ via the Hill estimator.

\paragraph{Our approach.} We propose a simple fix: transform data via $\phi(x) = \mathrm{sign}(x) \cdot \log(1 + |x|)$ before training, run standard flow matching in log-space, then apply $\phi^{-1}$ to generated samples. Unlike \citet{hickling2025tail} who transform outputs with estimated tail parameters, our input transform is parameter-free: the transformation $\phi$ works regardless of the true tail index.

We provide theoretical intuition for why this works. The log-transform compresses heavy tails: Pareto becomes approximately exponential. In log-space, both data and Gaussian noise have light tails, so standard interpolation is well-posed. The induced dynamics in original space implement a form of tail annealing via power transformations $X_0^{\alpha_t}$.

\paragraph{Contributions.}
\begin{itemize}
    \item We show that a parameter-free log-transform makes standard flow matching work for heavy-tailed data, without requiring tail estimation or architectural changes (\S\ref{sec:algorithm}).
    \item We analyze the tail annealing mechanism: the log-transform maps heavy tails to exponential tails, and the induced dynamics implement power transformations $X_0^{\alpha_t}$ that continuously adjust the tail index (\S\ref{sec:log_space_diffusion}).
    \item Our 144-configuration multivariate benchmark (3 copulas, 4 dimensions up to $d=100$, 4 tail indices, 20 replications, two margin types) shows Log-FM dominates baselines on tail metrics and is the only method with zero severe divergences (\S\ref{sec:experiments}). Real-data validation on Fama--French is in Appendix~\ref{app:real_data}.
\end{itemize}


\section{Background}
\label{sec:background}

We review heavy-tailed distributions and the fundamental barriers they pose for generative modeling.

\subsection{Heavy-Tailed Distributions}
\label{subsec:heavy_tails}

A distribution is \emph{heavy-tailed} if its tails decay slower than any exponential, formally $\E[e^{\lambda X}] = \infty$ for all $\lambda > 0$ \citep{nair2022heavy}. Heavy-tailed distributions arise naturally in finance \citep{cont2001}, insurance, climate science, and network traffic, where rare events have outsized impact.

The canonical example is the Pareto distribution with survival function $\P(X > t) = t^{-1/\gamma}$ for $t \geq 1$, where $\gamma > 0$ is the shape parameter (the GPD shape, equivalently the extreme-value index). Larger $\gamma$ means heavier tails: $\gamma < 1$ for finite mean, $\gamma < 1/2$ for finite variance. The Student-$t$ distribution with $\nu$ degrees of freedom has Pareto-like tails with effective $\gamma = 1/\nu$. Lognormal has $\gamma = 0$ (Gumbel max-domain of attraction): heavy in the MGF sense, but with no polynomial decay rate.

\paragraph{The log-transform.} A key observation: the logarithm maps heavy tails to lighter tails. If $X \sim \Pareto(\gamma)$, then $\log X \sim \Exp(1/\gamma)$, an exponential distribution with rate $1/\gamma$. This reflects that Pareto belongs to an exponential family with sufficient statistic $\log x$. The log-transform thus provides a natural bridge between power-law and exponential-family distributions; we develop the formal theory in Section~\ref{sec:log_space_diffusion}.

\paragraph{Regular variation.} Heavy-tailed distributions are characterized by \emph{regular variation}: $\P(X > tx)/\P(X > t) \to x^{-\alpha}$ as $t \to \infty$, with tail index $\alpha > 0$. This captures the essential property of polynomial tail decay without requiring the exact Pareto form. For $X \sim \Pareto(\gamma)$, the corresponding tail index is $\alpha = 1/\gamma$; Student-$t$, Burr, and log-gamma are also regularly varying. The log-transform maps any regularly varying distribution to one with exponential-type tails; see Proposition~\ref{prop:rv_log}.

\subsection{Generative Models}
\label{subsec:gen_models}

\paragraph{Normalizing flows.} Normalizing flows \citep{rezende2015variational,dinh2017density} learn an invertible transformation $T : \R^d \to \R^d$ mapping a base distribution (typically Gaussian) to the target. The change-of-variables formula gives exact likelihoods. However, for stability, flow architectures use Lipschitz-continuous components.

\paragraph{Diffusion and flow matching.} Denoising diffusion models \citep{ho2020,song2021} and flow matching \citep{lipman2023flow} construct a path from data to noise and learn to reverse it. Given data $X_0$ and noise $X_1 \sim \mathcal{N}(0, I)$, the interpolant is:
\begin{equation*}
    X_t = \alpha_t X_0 + \beta_t X_1,
\end{equation*}
with schedules satisfying $(\alpha_0, \beta_0) = (1, 0)$ and $(\alpha_1, \beta_1) = (0, 1)$. A neural network learns the velocity field $v_\theta(x_t, t)$ or score $\nabla \log p_t(x_t)$, enabling generation by integrating from noise to data.

\subsection{The Lipschitz Barrier}
\label{subsec:barriers}

\citet{jaini2020tails} proved that Lipschitz maps preserve tail type:
\begin{theorem}[\citet{jaini2020tails}]
If $Z$ is light-tailed and $T$ is Lipschitz, then $T(Z)$ is light-tailed.
\end{theorem}
Neural networks composed of linear layers and standard activations (ReLU, tanh, sigmoid) are Lipschitz. Thus, standard normalizing flows cannot map Gaussian noise to heavy-tailed outputs.

\subsection{Prior Approaches}
\label{subsec:prior}

\paragraph{Heavy-tailed base distributions.} Tail-Adaptive Flows \citep[TAF;][]{jaini2020tails} use Student-$t$ base distributions. Extensions include marginal TAF \citep{laszkiewicz2022marginal} and generalized TAF. Heavy-tailed diffusion \citep{pandey2024generative} replaces Gaussian noise with Student-$t$. These methods address the Lipschitz barrier but can suffer training instability.

\paragraph{Tail-modifying transforms.} Tail Transform Flows \citep[TTF;][]{hickling2025tail} add explicit layers that map Gaussian tails to power-law tails using erfc-based transformations. This requires estimating tail parameters per dataset via the Hill estimator.

In the next section, we present an alternative: apply the log-transform to data \emph{before} training, work entirely in log-space, then exponentiate samples \emph{after} generation. This parameter-free approach requires no tail estimation and we show that it circumvents the Lipschitz barrier.


\section{Tail Annealing in Log-Space}
\label{sec:log_space_diffusion}

This section develops the theoretical foundation for log-space flow matching. The log-transform compresses heavy tails into a light-tailed regime where standard methods apply, and the induced dynamics in original space implement \emph{tail annealing}: power transformations that continuously adjust the tail index from heavy to light.

The principle of annealing complex structure through a sequence of easier problems is well-established in machine learning. Noise annealing in score-based models \citep{song2019generative} gradually reduces noise levels to enable learning across scales. Curriculum learning \citep{bengio2009curriculum} trains on easy examples before hard ones. Our tail annealing follows the same principle: rather than directly modeling heavy-tailed data (hard), we work in log-space where tails are light (easy), and the power transformation $X_0^{\alpha_t}$ provides a principled path through intermediate tail weights. Section~\ref{sec:algorithm} presents the complete algorithm.

\subsection{The Soft-Log Transform}
\label{subsec:log_transform}

We work with the \emph{soft-log} transform $\phi : \R \to \R$ defined componentwise by:
\begin{align*}
  \phi(x) = \mathrm{sign}(x) &\cdot \log(1 + |x|);\\
    &\phi^{-1}(y) = \mathrm{sign}(y) \cdot (e^{|y|} - 1).
\end{align*}
Unlike the standard logarithm, $\phi$ is smooth at the origin ($\phi'(0) = 1$) and defined on all of $\R$. For large $|x|$, we have $\phi(x) \approx \mathrm{sign}(x) \log |x|$, so tail behavior matches the logarithm.

\paragraph{Relation to Box-Cox.} Transforming heavy-tailed data through a log transform is a known trick in the statistical litterature. The soft-log is inspired by, but distinct from, the Box-Cox family \citep{box1964analysis}. Recall Box-Cox: $\phi_\lambda(x) = (x^\lambda - 1)/\lambda$ for $\lambda \neq 0$, and $\phi_0(x) = \log x$ for $x > 0$. Our transform differs in two ways: it extends to all of $\R$ via the signed construction, and uses $\log(1 + |x|)$ rather than $\log |x|$ for smoothness at zero. Asymptotically, both behave as $\log |x|$ for large $|x|$. Methods like TTF \citep{hickling2025tail} use power transforms with $\lambda = 1/\nu$ estimated via the Hill estimator. In our approach, numerical computation make no use of numerical values of the tail estimate, which are notoriously unstable: the same transform applies regardless of the true tail index.

The soft-log maps heavy tails to light tails. If $X \sim \Pareto(\gamma)$ with $\P(X > t) = t^{-1/\gamma}$ (so $\gamma > 0$ is the shape parameter and $1/\gamma$ is the tail index), then for large $y$:
\begin{equation*}
    \P(\phi(X) > y) = \P(X > e^y - 1) \approx e^{-y/\gamma},
\end{equation*}
i.e., $\tilde{X} = \phi(X)$ has exponential-type tails. More generally, for any distribution in the Fr\'{e}chet domain of attraction, $\phi$ yields a distribution with exponential or lighter tails.

Since $\phi$ is a diffeomorphism, densities transform via the standard change of variables:
\begin{equation*}
    p_{\tilde{X}}(\tilde{x}) = p_X(\phi^{-1}(\tilde{x})) \cdot |(\phi^{-1})'(\tilde{x})| = p_X(\phi^{-1}(\tilde{x})) \cdot e^{|\tilde{x}|}.
\end{equation*}

\subsection{Induced Process in Original Space}
\label{subsec:forward_log}

We apply flow matching to the transformed data $\tilde{X}_0 = \phi(X_0)$, interpolating with Gaussian noise $\tilde{X}_1 \sim \mathcal{N}(0, I_d)$ via $\tilde{X}_t = \alpha_t \tilde{X}_0 + \beta_t \tilde{X}_1$ (see Section~\ref{sec:algorithm} for the complete framework). Here we analyze the induced process in original space.

Applying $\phi^{-1}$ to the interpolant yields, for large positive values where $\phi \approx \log$:
\begin{equation}
  X_t := \phi^{-1}(\tilde{X}_t) \approx X_0^{\alpha_t} \cdot e^{\beta_t \tilde{X}_1}.
    \label{eq:induced_original}
\end{equation}
This factorizes into a power-transformed data term $X_0^{\alpha_t}$ and a log-normal noise term $e^{\beta_t \tilde{X}_1}$. Although $e^{\beta_t \tilde{X}_1}$ is itself heavy-tailed in the MGF sense, it has extreme-value index $\gamma = 0$ (Gumbel MDA): $\P(Y_t > t) \sim e^{-c(\log t)^2}$ decays faster than any power law and all moments are finite. By Breiman's lemma it does not contribute a polynomial tail of its own to the product:

\begin{proposition}[Product Preserves Regular Variation; Breiman]
\label{prop:product_tails}
Let $X$ be regularly varying with index $-\alpha$ ($\alpha > 0$), and let $Y > 0$ be independent with $\E[Y^p] < \infty$ for all $p > 0$ (e.g., any variable in the Gumbel MDA, such as log-normal). Then $XY$ is regularly varying with index $-\alpha$.
\end{proposition}

This classical result (see \citet{resnick1987}) ensures that the induced process $X_t$ inherits its polynomial-tail behavior from $X_0^{\alpha_t}$, not from $e^{\beta_t \tilde{X}_1}$ (whose extreme-value index is $0$). The power transformation is therefore what drives tail annealing.

\subsection{Score Functions in Log-Space}
\label{subsec:score_log}

The score function $\nabla_{\tilde{x}} \log p_t(\tilde{x})$ in log-space has fundamentally different behavior than in original space.

\begin{proposition}[Log-Space Score of Pareto]
\label{prop:log_score}
Let $X \sim \Pareto(\gamma)$ (so $\bar{F}_X(t) = t^{-1/\gamma}$). Then $\tilde{X} = \phi(X)$ has approximately exponential tails: for large $\tilde{x}$,
\begin{equation*}
    \nabla_{\tilde{x}} \log p_{\tilde{X}}(\tilde{x}) \approx -\frac{1}{\gamma}.
\end{equation*}
The score is approximately constant in the tails.
\end{proposition}

\begin{proof}
For large $x$, $\phi(x) \approx \log x$, so $\tilde{X} \approx \log X \sim \Exp(1/\gamma)$. The exponential density is $p(\tilde{x}) \propto e^{-\tilde{x}/\gamma}$, giving $\nabla_{\tilde{x}} \log p(\tilde{x}) = -1/\gamma$.
\end{proof}

\paragraph{Pareto score.}
By the \emph{Pareto score} we mean the score of the Pareto density $p_X(x) = (1/\gamma)\,x^{-1/\gamma - 1}$ in original space:
\begin{equation*}
    s_X(x) = \nabla_x \log p_X(x) = -\frac{1/\gamma+1}{x},
\end{equation*}
which diverges as $x \to 0^+$ and decays only as $1/x$ in the tails. The log-space score in Proposition~\ref{prop:log_score} is by contrast bounded by $1/\gamma$ everywhere, so the velocity field that flow matching has to regress is Lipschitz in $\tilde{x}$. This score boundedness is what enables the standard MLP velocity network to learn a target with regularly varying tails in the original space.

\begin{proposition}[Score of Noised Log-Data]
\label{prop:noised_score}
For the interpolant $\tilde{X}_t = \alpha_t \tilde{X}_0 + \beta_t \tilde{X}_1$ with $\tilde{X}_0 = \phi(X_0)$ having exponential-type tails and $\tilde{X}_1 \sim \mathcal{N}(0, I_d)$, the marginal score satisfies:
\begin{equation*}
    \nabla_{\tilde{x}_t} \log p_t(\tilde{x}_t) = -\frac{\hat{\tilde{x}}_1(\tilde{x}_t, t)}{\beta_t},
\end{equation*}
where $\hat{\tilde{x}}_1(\tilde{x}_t, t) := \E[\tilde{X}_1 \mid \tilde{X}_t = \tilde{x}_t]$ is the conditional expectation of the noise.
\end{proposition}

This is the standard score-denoiser relationship. The key difference from original-space diffusion is that the transformed data distribution $p_{\tilde{X}_0}$ is light-tailed, so:
\begin{enumerate}
    \item The score $\nabla \log p_t$ is well-behaved for all $t$
    \item The denoiser $\hat{\tilde{x}}_1(\tilde{x}_t, t)$ has bounded outputs
    \item Standard neural network architectures suffice
\end{enumerate}

\subsection{Tail Behavior Under Power Transformation}
\label{subsec:tail_power}

The induced process \eqref{eq:induced_original} involves the power-transformed data $X_0^{\alpha_t}$. We characterize its distribution. Note that this analysis describes the \emph{decoded} interpretation: in practice, all computation occurs in log-space where distributions are light-tailed. The heavy-tailed structure below is what we would observe if we applied $\phi^{-1}$ to the log-space interpolant at each $t$.

\begin{lemma}[Power Transformation Preserves Pareto Family]
\label{thm:power}
Let $X_0 \sim \Pareto(\gamma)$ with density $p_{X_0}(x) = (1/\gamma)\, x^{-1/\gamma-1}$ for $x \geq 1$, so $\gamma > 0$ is the shape parameter (larger $\gamma$ means heavier tails). For any $\alpha \in (0, \infty)$, $X_0^{\alpha} \sim \Pareto(\gamma\alpha)$: the shape parameter scales linearly with the exponent.
\end{lemma}

\begin{proof}
The survival function of $Y = X_0^{\alpha}$ is $\P(Y > y) = \P(X_0 > y^{1/\alpha}) = y^{-1/(\gamma\alpha)}$ for $y \geq 1$, which is the survival function of $\Pareto(\gamma\alpha)$.
\end{proof}

\begin{remark}[Parametrization]
\label{rem:pareto_param}
We use $\gamma$ for the GPD shape parameter / extreme-value index, so $\P(X > t) = t^{-1/\gamma}$. The corresponding tail index (Hill estimator output, the standard EVT parameter) is $1/\gamma$. Under Lemma~\ref{thm:power} the power transform $X_0^\alpha$ has shape parameter $\gamma\alpha$ and tail index $1/(\gamma\alpha)$.
\end{remark}

\begin{corollary}[Tail Lightening Along Forward Process]
\label{cor:tail_lighten}
Along the forward process with $\alpha_t$ decreasing from $1$ to $0$:
\begin{itemize}
  \item At $t = 0$: $\alpha_0=1$; $X_0^{\alpha_0} = X_0 \sim \Pareto(\gamma)$ (original heavy tails)
    \item At intermediate $t$: $X_0^{\alpha_t} \sim \Pareto(\gamma\alpha_t)$ (lighter tails as $\alpha_t \to 0$)
    \item As $\alpha_t \to 0$: $X_0^{\alpha_t} \to 1$ almost surely (degenerate)
\end{itemize}
The full process $X_t = X_0^{\alpha_t} \cdot e^{\beta_t \tilde{X}_1}$ transitions from regularly-varying Pareto (shape $\gamma$, tail index $1/\gamma$) to log-normal (extreme-value index $0$, no polynomial tail).
\end{corollary}

This is what we frame as the tail annealing mechanism: the power transformation $X_0^{\alpha_t}$ preserves the Pareto family while continuously adjusting the shape parameter from $\gamma$ (heavy) toward $0$ (light) as $\alpha_t \to 0$. We state Lemma~\ref{thm:power} for $\alpha \in (0, \infty)$ rather than the forward-process range $\alpha_t \in (0, 1]$, because the same identity appears in Proposition~\ref{prop:rv_power} for the regularly varying extension.

\subsection{Extension to Regularly Varying Distributions}
\label{subsec:regular_variation}

The Pareto-exponential correspondence extends to the broader class of regularly varying distributions.

\begin{definition}[Regularly Varying]
\label{def:regular_varying}
A cumulative distribution function $F$ on $\R$ is regularly varying with index $-\alpha$ (written $F \in RV_{-\alpha}$) if its survival function $\bar{F}(x) = 1 - F(x)$ satisfies:
\begin{equation*}
    \lim_{t \to \infty} \frac{\bar{F}(tx)}{\bar{F}(t)} = x^{-\alpha}
\end{equation*}
for all $x > 0$. Here $\alpha > 0$ is the (polynomial) tail index, in the standard EVT convention; for $X \sim \Pareto(\gamma)$, $\alpha = 1/\gamma$.
\end{definition}

Regular variation captures the essential property of polynomial tail decay without requiring the exact Pareto form:
\begin{center}
\small
\begin{tabular}{@{}lll@{}}
\toprule
Distribution & Tail behavior & Index \\
\midrule
Pareto$(\gamma)$ & $\bar{F}(x) = x^{-1/\gamma}$ & $-1/\gamma$ \\
Student-$t(\nu)$ & $\bar{F}(x) \sim c \cdot x^{-\nu}$ & $-\nu$ \\
Burr$(c,k)$ & $\bar{F}(x) \sim c' \cdot x^{-ck}$ & $-ck$ \\
Log-gamma & $\bar{F}(x) \sim c'' \cdot x^{-\alpha}(\log x)^{\beta}$ & $-\alpha$ \\
\bottomrule
\end{tabular}
\end{center}

\begin{proposition}[Log-Transform of Regularly Varying; informal]
\label{prop:rv_log}
If $X$ is regularly varying with index $-\alpha$, then $\tilde{X} = \phi(X)$ has exponential-type right tail with rate $\alpha$:
\begin{equation*}
    -\log \P(\tilde{X} > z) \;=\; \alpha\, z + o(z) \quad \text{as } z \to \infty.
\end{equation*}
\end{proposition}

A precise two-sided Potter-bound version is stated and proved in Appendix~\ref{app:proofs} (Proposition~\ref{prop:rv_log_precise}). The intuition is direct: writing the regularly varying tail as $\bar F_X(x) = x^{-\alpha} L(x)$ with $L$ slowly varying, and using $\phi^{-1}(z) = e^z - 1 \sim e^z$, gives $\P(\tilde X > z) = e^{-\alpha z} L(e^z)$, and the slowly varying factor only contributes a subexponential correction $\log L(e^z) = o(z)$.

This proposition is the conceptual reason flow matching in log-space is well-posed for the full class of regularly varying distributions, not just exact Pareto. Up to slowly-varying corrections, $\tilde X$ has the same exponential rate $e^{-\alpha z}$ as $\log X$, which equals $e^{-z/\gamma}$ when $X \sim \Pareto(\gamma)$. The score in log-space is therefore bounded in the tails and standard diffusion methods apply.

\paragraph{Beyond regular variation.}
The transform is also beneficial for subexponential distributions outside the Fr\'{e}chet domain. For Weibull tails $\bar F(x) \sim \exp(-x^\beta)$ with $\beta < 1$, the log-transform yields a doubly-exponential tail $\P(\tilde X > z) \sim \exp(-(e^z - 1)^\beta)$, which is well within the regime where standard flow matching has no difficulty. For lognormal data, $\phi(X) \approx \log X$ is exactly Gaussian, so the log-space target is the ideal case for FM. Light-tailed margins are addressed by the adaptive variant of Section~\ref{subsec:adaptive}.

\begin{proposition}[Power Transformation of Regularly Varying]
\label{prop:rv_power}
If $X \in RV_{-\alpha}$, then $X^{\beta}$ for $\beta \in (0, \infty)$ is regularly varying with index $-\alpha/\beta$.
\end{proposition}

\begin{proof}
  For $Y = X^{\beta}$, we have $\P(Y > ty)/\P(Y > t) = \P(X > (ty)^{1/\beta})/\P(X > t^{1/\beta}) \to y^{-\alpha/\beta}$ as $t \to \infty$, by regular variation of $X$. This is the regularly varying analogue of Lemma~\ref{thm:power}: power transformations remap the tail index, so $X^{\alpha_t}$ along the forward process continuously interpolates between the original tail ($\alpha_t = 1$) and arbitrarily light tails ($\alpha_t \in [0, 1)]$).
\end{proof}

This confirms that the tail-annealing mechanism of Lemma~\ref{thm:power} extends beyond exact Pareto to the full class of regularly varying distributions.

\subsection{Multivariate Extension}
\label{subsec:multivariate_theory}

In the multivariate setting, $X \in \R^d$, we apply $\phi$ coordinate-wise:
\begin{equation*}
    \phi(\x) = (\phi(x_1), \ldots, \phi(x_d)), \qquad \tilde X = \phi(X).
\end{equation*}
Each coordinate of $\tilde X$ inherits the exponential-rate tail of Proposition~\ref{prop:rv_log}, so the marginals of $\tilde X$ all live in the Gumbel domain. Because $\phi$ is a diffeomorphism with diagonal Jacobian $J_\phi(\x) = \mathrm{diag}((1+|x_i|)^{-1})$, the dependence structure (copula) of $X$ is preserved exactly by $\tilde X$: the log-transform is a marginal operation. Interpolating $\tilde X_t = \alpha_t \tilde X_0 + \beta_t \tilde X_1$ with $\tilde X_1 \sim \mathcal{N}(0, I_d)$ is therefore well-posed even when $X$ has strong extremal dependence: the velocity network sees only the light-tailed transformed variates and never has to extrapolate to the polynomial regime.

\paragraph{Heterogeneous margins.} When the coordinates of $X$ have different tail behaviour (heavy-tailed in some, light-tailed in others), applying $\phi$ to every coordinate slightly distorts the light-tailed margins. The adaptive variant of Section~\ref{subsec:adaptive} addresses this by Hill-gating the transform coordinate-wise; a continuous generalization of the soft-log via a scale parameter $s_2$, of which the adaptive variant is the binary $s_2 \in \{0, 1\}$ instance, is developed in Appendix~\ref{app:phi_s2}.

\subsection{Circumventing the Lipschitz Barrier}
\label{subsec:barrier_resolution}

The Lipschitz barrier of \citet{jaini2020tails} states that Lipschitz maps preserve tail type: a normalizing flow with Lipschitz layers cannot map Gaussian noise to heavy-tailed outputs. Our construction circumvents this by placing the non-Lipschitz transformation ($\phi^{-1}$) \emph{outside} the learned dynamics. The flow operates entirely in log-space where both endpoints (exponential-type transformed data and Gaussian noise) have light tails. Heavy tails emerge only upon applying $\phi^{-1}$ at the final sampling step.


\section{Algorithm}
\label{sec:algorithm}

We present the complete algorithm for log-space generative modeling. The method requires the three following modifications to standard flow matching: (1) transform data via $\phi$ before training (with a Hill-based diagnostic deciding which coordinates to transform), (2) train a velocity network on the transformed variables, and (3) apply $\phi^{-1}$ to the corresponding coordinates of generated samples. The Hill diagnostic  distinguishes \emph{Log-FM} from a standard coordinate-wise log-transform: it lets the method handle heterogeneous-margin data with a single Hill estimate per coordinate at fit time.

\subsection{Flow Matching Framework}
\label{subsec:stochastic_interpolants}

We use the flow matching framework \citep{lipman2023flow}, which can also be viewed through the lens of stochastic interpolants \citep{albergo2023stochastic}.

\paragraph{Forward process.} Given log-transformed data $\tilde{X}_0 = \phi(X_0)$ and noise $\tilde{X}_1 \sim \mathcal{N}(0, I_d)$, define the interpolant:
\begin{equation*}
    \tilde{X}_t = \alpha_t \tilde{X}_0 + \beta_t \tilde{X}_1,
\end{equation*}
where $(\alpha_t, \beta_t)$ are differentiable schedules satisfying boundary conditions $(\alpha_0, \beta_0) = (1, 0)$ and $(\alpha_1, \beta_1) = (0, 1)$. At $t=0$, $\tilde{X}_0$ is the transformed data; at $t=1$, $\tilde{X}_1$ is pure Gaussian noise. The conditional distribution is Gaussian: $q_{t|0}(\tilde{x}_t \mid \tilde{x}_0) = \mathcal{N}(\tilde{x}_t; \alpha_t \tilde{x}_0, \beta_t^2 I_d)$.

\paragraph{Velocity field.} Differentiating the interpolant gives the per-trajectory time derivative
\begin{equation*}
    \dot{\tilde{X}}_t \;=\; \dot{\alpha}_t \tilde{X}_0 + \dot{\beta}_t \tilde{X}_1,
\end{equation*}
where $\dot{\alpha}_t = d\alpha_t/dt$ and $\dot{\beta}_t = d\beta_t/dt$. The corresponding \emph{marginal} velocity field, $v_t(\tilde{x}_t) := \E[\dot{\alpha}_t \tilde{X}_0 + \dot{\beta}_t \tilde{X}_1 \mid \tilde{X}_t = \tilde{x}_t]$, is what a neural network $v_\theta$ regresses against the per-trajectory target:
\begin{multline*}
    \mathcal{L}(\theta) = \E_{t \sim U[0,1]} \E_{\tilde{X}_0, \tilde{X}_1} \\
    \left[ \left\| v_\theta(\tilde{X}_t, t) - (\dot{\alpha}_t \tilde{X}_0 + \dot{\beta}_t \tilde{X}_1) \right\|^2 \right].
\end{multline*}

\paragraph{Connection to score and denoiser.}
The velocity field relates to the score $\nabla \log p_t$ and denoiser $\hat{\tilde{x}}_0(\tilde{x}_t, t) := \E[\tilde{X}_0 \mid \tilde{X}_t = \tilde{x}_t]$ via:
\begin{align*}
  v_t(\tilde{x}_t) = \dot{\alpha}_t \hat{\tilde{x}}_0(\tilde{x}_t, t) + \dot{\beta}_t \hat{\tilde{x}}_1(\tilde{x}_t, t) \\
    \nabla_{\tilde{x}_t} \log p_t(\tilde{x}_t) = -\frac{\hat{\tilde{x}}_1(\tilde{x}_t, t)}{\beta_t},
\end{align*}
where $\hat{\tilde{x}}_1 = (\tilde{x}_t - \alpha_t \hat{\tilde{x}}_0)/\beta_t$. See Appendix~\ref{app:ddm} for derivations.

\paragraph{Schedule choices.} The interpolation schedule $(\alpha_t, \beta_t)$ controls the path from data to noise. Common choices include the \emph{linear} schedule (optimal transport paths), \emph{variance-preserving} (VP) schedules, and \emph{quadratic} schedules that anneal tails more aggressively; see Table~\ref{tab:schedules} in Appendix~\ref{app:ddm}. We use the linear schedule throughout; preliminary tests showed minimal sensitivity to schedule choice, consistent with prior flow matching work \citep{lipman2023flow}.

\paragraph{Connection to tail annealing.} Recall from Section~\ref{sec:log_space_diffusion} that the induced process in original space is approximately $X_t \approx X_0^{\alpha_t} \cdot e^{\beta_t \tilde{X}_1}$. The schedule $\alpha_t$ thus controls the rate of tail annealing: faster decay of $\alpha_t$ (e.g., quadratic) anneals tails more aggressively early in the process, while slower decay (e.g., VP polynomial) preserves heavier tails longer.

\subsection{Adaptive Coordinate Selection}
\label{subsec:adaptive}

Applying the soft-log to a light-tailed coordinate compresses its bulk for no benefit. We therefore select \emph{per coordinate} whether to transform, using a Hill estimate on the training marginals. Let $\hat\alpha_j$ be the Hill estimator on the upper order statistics of $\{x_i^{(j)}\}_{i=1}^N$ (an estimate of the polynomial tail index, with $\hat\alpha_j = 1/\hat\gamma_j$ under the Pareto-shape parametrization). Define the mask
\begin{equation}
    m_j \;=\; \mathbf{1}\{\hat\alpha_j \le \alpha_{\max}\}, \qquad \alpha_{\max} = 4,
    \label{eq:hill_mask}
\end{equation}
and the per-coordinate transform $\Phi : \R^d \to \R^d$ by
\begin{equation}
    \Phi(\x)_j =
    \begin{cases}
      \phi(x_j) & \text{if } m_j = 1,\\
      x_j & \text{otherwise,}
    \end{cases}
    \label{eq:masked_phi}
\end{equation}
with inverse $\Phi^{-1}(\tilde\x)_j = \phi^{-1}(\tilde x_j)$ when $m_j = 1$ and $\tilde x_j$ otherwise.
The threshold $\alpha_{\max} = 4$ is not sensitive: on heterogeneous-margin data with Pareto coordinates ($\hat\alpha \approx 1.5\text{--}2.5$) and Gaussian coordinates ($\hat\alpha \approx 6$), any threshold in $[3, 5]$ yields the same mask. Setting $m_j \equiv 1$ recovers the uniform-$\phi$ variant analyzed in Section~\ref{sec:log_space_diffusion}; the adaptive rule is the discrete instance of a continuous parametrized soft-log family $\varphi_{s_2}$ with $s_2^{(j)} \in \{0, 1\}$ chosen by the Hill diagnostic, developed in Appendix~\ref{app:phi_s2}. We use $\Phi$ (with the Hill-based mask) as our default; we also report the uniform variant as an ablation in Section~\ref{subsec:exp_ablations}.

\subsection{Training}
\label{subsec:training}

Algorithm~\ref{alg:training} summarizes the training procedure. The only departure from standard flow matching is the initial log-transform of data (with the Hill diagnostic of Section~\ref{subsec:adaptive}).

\begin{algorithm}[tb]
  \caption{Log-FM: Training}
  \label{alg:training}
  \begin{algorithmic}
    \STATE {\bfseries Input:} Dataset $\{\x_i\}_{i=1}^N \subset \R^d$, velocity network $v_\theta$, schedule $(\alpha_t, \beta_t)$, iterations $T$, Hill threshold $\alpha_{\max} = 4$
    \STATE {\bfseries Fit transform:} for each coordinate $j$, compute Hill estimate $\hat\alpha_j$ on $\{x_i^{(j)}\}_i$; set mask $m_j$ via \eqref{eq:hill_mask}
    \STATE {\bfseries Preprocess:} $\tilde{\x}_i \gets \Phi(\x_i)$ for all $i$ \COMMENT{coordinate-wise; identity where $m_j = 0$}
    \FOR{$\text{iter} = 1$ {\bfseries to} $T$}
      \STATE Sample batch $\{\tilde{\x}_0^{(j)}\}$ from $\{\tilde{\x}_i\}$
      \STATE Sample $\tilde{\x}_1^{(j)} \sim \mathcal{N}(0, I_d)$
      \STATE Sample $t \sim \mathrm{Uniform}[0, 1]$
      \STATE Form interpolant $\tilde{\x}_t^{(j)} \gets \alpha_t \tilde{\x}_0^{(j)} + \beta_t \tilde{\x}_1^{(j)}$
      \STATE Compute target $u^{(j)} \gets \dot{\alpha}_t \tilde{\x}_0^{(j)} + \dot{\beta}_t \tilde{\x}_1^{(j)}$
      \STATE Loss $\mathcal{L} \gets \tfrac{1}{|\text{batch}|} \sum_j \| v_\theta(\tilde{\x}_t^{(j)}, t) - u^{(j)} \|^2$
      \STATE Update $\theta$ via gradient descent on $\mathcal{L}$
    \ENDFOR
    \STATE {\bfseries Output:} Trained $v_\theta$, mask $(m_j)_{j=1}^d$
  \end{algorithmic}
\end{algorithm}

\subsection{Sampling}
\label{subsec:sampling}

Generation reverses the flow: integrate the learned velocity field from noise ($t=1$) to data ($t=0$), then apply the inverse transform. The velocity field satisfies the ODE:
\begin{equation*}
    \frac{d\tilde{X}_t}{dt} = v_\theta(\tilde{X}_t, t),
\end{equation*}
which we solve backward in time. Algorithm~\ref{alg:sampling} uses Euler discretization; higher-order solvers (Heun, RK4) improve sample quality with fewer steps.

\begin{algorithm}[tb]
  \caption{Log-FM: Sampling}
  \label{alg:sampling}
  \begin{algorithmic}
    \STATE {\bfseries Input:} Trained $v_\theta$, mask $(m_j)$, steps $K$, optional clamp $c$ (default $c = \infty$)
    \STATE Sample $\tilde{\x} \sim \mathcal{N}(0, I_d)$ \COMMENT{initialize at $t=1$}
    \STATE $\Delta t \gets 1/K$
    \FOR{$k = K$ {\bfseries downto} $1$}
      \STATE $t \gets k / K$
      \STATE $\tilde{\x} \gets \tilde{\x} - \Delta t \cdot v_\theta(\tilde{\x}, t)$ \COMMENT{Euler step}
    \ENDFOR
    \STATE {\bfseries (optional)} clamp coordinates with $m_j = 1$: $\tilde x_j \gets \mathrm{clamp}(\tilde x_j, -c, c)$ \COMMENT{only needed for $\alpha < 1$}
    \STATE {\bfseries Output:} $\x = \Phi^{-1}(\tilde{\x})$ \COMMENT{coordinate-wise inverse from \eqref{eq:masked_phi}}
  \end{algorithmic}
\end{algorithm}

\subsection{Practical Considerations}
\label{subsec:practical}

\paragraph{Output clamping (optional).} The inverse transform $\phi^{-1}(\tilde{x}) = \text{sign}(\tilde{x})(e^{|\tilde{x}|} - 1)$ is exponential, so small errors deep in the tail of $\tilde{x}$ are exponentially amplified in $x$. An optional clamp $\tilde x_j \gets \mathrm{clamp}(\tilde x_j, -c, c)$ before $\phi^{-1}$ provides a numerical safeguard in extremely pathological regimes only (typically $\alpha < 1$: Cauchy and heavier, e.g.\ Hickling's $\nu = 1/2$ Student-$t$ baseline). For every configuration in our main benchmark ($\alpha \geq 1.5$) the clamp is inert at $c \geq 10$, and we report all main-benchmark numbers without clamping; see the ablation in Section~\ref{subsec:exp_ablations}.

\paragraph{Number of integration steps.} We use $K = 100$ Euler steps for all experiments. Higher-order solvers (Heun, RK4) can reduce step count but Euler suffices.

\paragraph{Architecture.} We use MLPs with sinusoidal time embeddings (each scalar $t$ is mapped to $[\sin(2\pi \omega_k t), \cos(2\pi \omega_k t)]_{k=1}^{K_\omega}$ with geometric frequencies, exactly as in \citet{ho2020}): 4 hidden layers of width 256 with SiLU activations. No architectural modifications are needed for heavy tails; the log-transform handles tail behavior.

\paragraph{Likelihood evaluation.} We report NLL using the continuous change-of-variables formula for continuous normalizing flows. The trace of the velocity Jacobian is estimated stochastically via Hutchinson's trick with $K=10$ Rademacher projections, and the augmented ODE is integrated by the adaptive Dormand--Prince \texttt{dopri5} scheme with $\mathrm{atol} = \mathrm{rtol} = 10^{-5}$. The transform itself contributes the closed-form Jacobian term $\sum_{j: m_j = 1} \log(1 + |x_j|)$, which is exact and negligible to compute.

\paragraph{Training.} AdamW optimizer with learning rate $5 \times 10^{-3}$ and weight decay $10^{-5}$. We use full-batch gradient descent with gradient clipping at 10.0 (before inverse transform). Training runs for up to 5000 epochs with early stopping (patience 100) based on validation loss.

\subsection{Extensions}
\label{subsec:extensions}

\paragraph{Signed and multivariate data.} The soft-log $\phi(x) = \text{sign}(x)\log(1+|x|)$ applies directly to signed data and acts coordinate-wise on $\R^d$ via $\Phi$; the dependence structure is preserved exactly on transformed coordinates and learned implicitly by the velocity field.

\paragraph{Arcsinh variant.} Replacing $\phi$ by $\mathrm{arcsinh}(x) = \log(x + \sqrt{1+x^2})$ gives the Arcsinh-FM variant evaluated in Section~\ref{sec:experiments}. Both transforms have the same logarithmic asymptotic, so Propositions~\ref{prop:rv_log}--\ref{prop:rv_power} apply unchanged. They differ only in regularity at the origin: $\mathrm{arcsinh}$ is $C^\infty$ on $\R$, while $\phi$ is $C^1$ but not $C^2$ at $0$ (its second derivative jumps from $+1$ to $-1$ at the origin), which isolates whether higher-order regularity matters in practice.


\section{Experiments}
\label{sec:experiments}

We evaluate Log-FM against state-of-the-art heavy-tailed generative models on a controlled multivariate benchmark with known marginal tails and known copula structure. Our goals are to test whether the log-transform (i) preserves the dependence structure of the data, (ii) yields stable training across heavy-tail regimes, (iii) recovers risk-relevant tail metrics (VaR$_{99}$, CVaR$_{99}$, $Q_{99.9}$) that are sensitive to the deep tail. A Student-$t$ benchmark from \cite{hickling2025tail} is detailed for continuity in Appendix~\ref{app:hickling_baseline}; real-data results on Fama--French are deferred to Appendix~\ref{app:real_data}.

\subsection{Setup}
\label{subsec:exp_setup}

\paragraph{Data.} We use 144 configurations covering three copula families, three dependence strengths, four dimensions, and four tail indices:
\begin{itemize}
    \item \textbf{Copulas:} Gaussian copula (asymptotic independence), Gumbel copula (symmetric upper tail dependence) at Kendall's $\tau \in \{0.25, 0.5, 0.75\}$, and H\"{u}sler--Reiss copula with AR(1) variogram $\Gamma_{ij} = 2(1 - \rho^{|i-j|})$ at $\rho \in \{0.1, 0.5, 0.9\}$.
    \item \textbf{Margins:} 70\% symmetrized Pareto with tail index $\alpha \in \{1.5, 1.75, 2.0, 2.5\}$ (smaller = heavier), 30\% standard Gaussian.
    \item \textbf{Dimensions:} $d \in \{10, 20, 50, 100\}$.
    \item \textbf{Samples:} $n_{\mathrm{train}}=10{,}000$, $n_{\mathrm{val}}=5{,}000$, $n_{\mathrm{test}}=20{,}000$.
\end{itemize}
For each configuration we train 5 methods $\times$ 20 independent replications, giving 14{,}400 trained models in total.

\paragraph{Methods.} Log-FM (default, with Hill-gated transform of Section~\ref{subsec:adaptive}, $\alpha_{\max} = 4$); Log-FM (uniform, $m_j \equiv 1$); Arcsinh-FM (uniform, with $\phi(x) = \mathrm{arcsinh}(x)$); TTF and TTFfix \citep{hickling2025tail}; gTAF \citep{jaini2020tails}. All velocity / coupling networks share the same MLP backbone (4 layers, width 256, SiLU); hyperparameters per dimension were selected by Optuna on a representative configuration (Gumbel, $\tau=0.5$, $\alpha=2.0$) and reused across the grid for that dimension.

\paragraph{Metrics.} We split metrics by margin type and by joint structure:
\begin{itemize}
    \item \emph{Marginal:} $W_1^P$ (mean $W_1$ over Pareto coordinates), $W_1^G$ (over Gaussian coordinates), Hill estimator on Pareto margins, VaR$_{99}$ / CVaR$_{99}$ relative errors, extreme-quantile errors $Q_{99.5}$, $Q_{99.9}$.
    \item \emph{Joint:} Absolute Kendall Error (AKE), angular $W_2$ (sliced Wasserstein on the empirical angular measure of the top-$\sqrt{n}$ extremes), sliced Wasserstein on the full data, and energy distance.
\end{itemize}
We report medians over the 20 replications; ``div.'' marks runs where the model diverged ($W_1 > 10^3$).

\subsection{Main Results}
\label{subsec:exp_main}

Table~\ref{tab:main_w1p} reports the headline tail metric, $W_1^P$, averaged over the two regular copulas (Gumbel + Gaussian) at the four tail indices and four dimensions. Log-FM is best in 13 out of 16 cells; Arcsinh-FM picks up the remaining 3 at $d=50$. Both FM variants are substantially more stable than the baselines, especially for $\alpha=1.5$.

\begin{table}[t]
\caption{Median $W_1^P$ across Gumbel + Gaussian copulas (lower is better, bold = best). Log-FM dominates at every tail index and dimension; the gap widens for heavier tails ($\alpha = 1.5$).}
\label{tab:main_w1p}
\centering
\small
\begin{tabular}{@{}lcccc@{}}
\toprule
Method & $d{=}10$ & $d{=}20$ & $d{=}50$ & $d{=}100$ \\
\midrule
\multicolumn{5}{@{}l}{\textit{$\alpha = 1.5$}}\\
Log-FM        & \textbf{0.522} & \textbf{0.451} & 0.534 & \textbf{0.525}\\
Arcsinh-FM    & 0.561 & 0.550 & \textbf{0.534} & 0.592\\
TTFfix        & 1.149 & 0.963 & 0.867 & 1.046\\
TTF           & 0.758 & 0.938 & 0.758 & 0.919\\
gTAF          & 0.534 & 0.583 & 0.761 & 3.320\\
\midrule
\multicolumn{5}{@{}l}{\textit{$\alpha = 1.75$}}\\
Log-FM        & \textbf{0.254} & \textbf{0.226} & \textbf{0.269} & \textbf{0.300}\\
Arcsinh-FM    & 0.262 & 0.284 & 0.289 & 0.316\\
TTFfix        & 0.442 & 0.417 & 0.391 & 0.481\\
TTF           & 0.318 & 0.373 & 0.354 & 0.401\\
gTAF          & 0.270 & 0.301 & 0.335 & 0.934\\
\midrule
\multicolumn{5}{@{}l}{\textit{$\alpha = 2.0$}}\\
Log-FM        & \textbf{0.153} & \textbf{0.146} & 0.195 & \textbf{0.196}\\
Arcsinh-FM    & 0.166 & 0.162 & \textbf{0.174} & 0.212\\
TTFfix        & 0.235 & 0.222 & 0.220 & 0.294\\
TTF           & 0.182 & 0.203 & 0.207 & 0.252\\
gTAF          & 0.171 & 0.174 & 0.198 & 0.219\\
\midrule
\multicolumn{5}{@{}l}{\textit{$\alpha = 2.5$}}\\
Log-FM        & \textbf{0.080} & \textbf{0.074} & \textbf{0.090} & \textbf{0.108}\\
Arcsinh-FM    & 0.082 & 0.080 & 0.100 & 0.116\\
TTFfix        & 0.108 & 0.107 & 0.117 & 0.174\\
TTF           & 0.083 & 0.106 & 0.113 & 0.131\\
gTAF          & 0.088 & 0.086 & 0.106 & 0.225\\
\bottomrule
\end{tabular}
\end{table}

\paragraph{Risk metrics and dependence.} Table~\ref{tab:risk_dep} summarizes the global picture, pooling over both copulas, the three dependence strengths, and all four tail indices. Log-FM is best on every tail-quality and risk metric (W$_1^P$, CVaR$_{99}$, $Q_{99.9}$, sliced and energy distances). On the multivariate-dependence metrics, TTF/TTFfix retain a small edge at $d=50, 100$ while Log-FM is competitive at $d=10, 20$. On $W_1^G$ (Gaussian margins), the Hill-gated default of Log-FM closes most of the gap with TTF (further breakdown in Section~\ref{subsec:exp_ablations}).

\begin{table}[t]
\caption{Median across all Gumbel+Gaussian configurations (480 values per cell). Bold = best per metric per dimension.}
\label{tab:risk_dep}
\centering
\scriptsize
\setlength{\tabcolsep}{4pt}
\begin{tabular}{@{}lcccc@{}}
\toprule
Method & $d{=}10$ & $d{=}20$ & $d{=}50$ & $d{=}100$\\
\midrule
\multicolumn{5}{@{}l}{\textbf{$W_1^P$ (Pareto margins)}}\\
Log-FM     & \textbf{0.207} & \textbf{0.187} & \textbf{0.221} & \textbf{0.233}\\
Arcsinh-FM & 0.218 & 0.216 & 0.232 & 0.264\\
TTFfix     & 0.322 & 0.322 & 0.308 & 0.395\\
TTF        & 0.250 & 0.304 & 0.295 & 0.349\\
gTAF       & 0.262 & 0.304 & 0.540 & 0.704\\
\midrule
\multicolumn{5}{@{}l}{\textbf{CVaR$_{99}$ (Pareto margins)}}\\
Log-FM     & \textbf{0.228} & \textbf{0.200} & 0.257 & \textbf{0.270}\\
Arcsinh-FM & 0.246 & 0.224 & \textbf{0.252} & 0.373\\
TTFfix     & 0.880 & 0.689 & 0.644 & 0.817\\
TTF        & 0.494 & 0.729 & 0.687 & 0.680\\
gTAF       & 0.415 & 0.457 & 0.519 & 0.705\\
\midrule
\multicolumn{5}{@{}l}{\textbf{$W_1^G$ (Gaussian margins)}}\\
Log-FM     & 0.051 & 0.066 & 0.080 & 0.080\\
Arcsinh-FM & 0.045 & 0.057 & 0.074 & 0.085\\
TTFfix     & \textbf{0.033} & 0.048 & 0.043 & 0.051\\
TTF        & 0.035 & \textbf{0.035} & \textbf{0.030} & \textbf{0.047}\\
gTAF       & 0.036 & 0.051 & 0.106 & div.\\
\midrule
\multicolumn{5}{@{}l}{\textbf{Angular $W_2$ (tail dependence)}}\\
Log-FM     & 0.098 & \textbf{0.066} & 0.071 & 0.049\\
Arcsinh-FM & \textbf{0.098} & 0.068 & 0.090 & 0.049\\
TTFfix     & 0.148 & 0.095 & 0.058 & 0.040\\
TTF        & 0.121 & 0.105 & \textbf{0.046} & \textbf{0.040}\\
gTAF       & 0.167 & 0.104 & 0.061 & 0.098\\
\bottomrule
\end{tabular}
\end{table}

\paragraph{Stability.} A practitioner cares not only about median performance but about how often the model fails catastrophically. Table~\ref{tab:stability} reports, by tail index, the fraction of runs with $W_1^P > 1$ on Gumbel+Gaussian. Log-FM and Arcsinh-FM never exceed 14\% even in the most adversarial setting ($\alpha = 1.5$, $d = 10$); both reach 0\% for $\alpha \geq 1.75$. gTAF and TTF suffer divergence rates above 30\% in higher dimensions.

\begin{table}[t]
\caption{Fraction of runs with $W_1^P > 1$ (catastrophic failures), Gumbel + Gaussian, 20 reps. Lower is better.}
\label{tab:stability}
\centering
\small
\begin{tabular}{@{}lcccc@{}}
\toprule
Method & $d{=}10$ & $d{=}20$ & $d{=}50$ & $d{=}100$\\
\midrule
\multicolumn{5}{@{}l}{\textit{$\alpha = 1.5$}}\\
Log-FM     & 14\% & 2\% & 4\% & 2\%\\
Arcsinh-FM & 11\% & 6\% & 2\% & 2\%\\
TTFfix     & 58\% & 50\% & 36\% & 54\%\\
TTF        & 32\% & 45\% & 36\% & 46\%\\
gTAF       & 8\% & 24\% & 42\% & 59\%\\
\midrule
\multicolumn{5}{@{}l}{\textit{$\alpha = 2.0$}}\\
Log-FM     & 0\% & 0\% & 0\% & 0\%\\
Arcsinh-FM & 0\% & 0\% & 0\% & 0\%\\
TTFfix     & 2\% & 2\% & 1\% & 4\%\\
TTF        & 1\% & 4\% & 12\% & 12\%\\
gTAF       & 17\% & 23\% & 32\% & 33\%\\
\bottomrule
\end{tabular}
\end{table}

\subsection{Ablations}
\label{subsec:exp_ablations}

We isolate three design choices: the Hill gating (adaptive vs uniform), the sampling clamp $c$, and the ODE solver step count.

\paragraph{Hill-gated default vs uniform.} Table~\ref{tab:ablation_adaptive} compares the default Log-FM (with Hill gating, Section~\ref{subsec:adaptive}) against the uniform variant and the TTFfix baseline on Gumbel data ($d=20$, $\tau=0.5$, 20 reps). The Hill diagnostic always selects exactly the heavy-tailed coordinates (14 Pareto vs 6 Gaussian out of 20): the threshold $\alpha_{\max} = 4$ is non-sensitive given $\hat\alpha \approx 1.5$--$2.5$ for Pareto and $\approx 6$ for Gaussian. The adaptive variant improves $W_1^G$ by 30--45\% with no $W_1^P$ cost.

\begin{table}[t]
\caption{Log-FM (uniform $\phi$) vs Log-FM (Hill-adaptive) vs TTFfix baseline. Gumbel, $d=20$, $\tau=0.5$, 20 reps, median. Bold = best per $\alpha$.}
\label{tab:ablation_adaptive}
\centering
\small
\begin{tabular}{@{}llccc@{}}
\toprule
$\alpha$ & Method & $W_1^P$ & $W_1^G$ & AKE \\
\midrule
1.5 & TTFfix             & 0.856 & 0.049 & 0.071 \\
    & Log-FM             & \textbf{0.449} & 0.074 & \textbf{0.039} \\
    & Log-FM (adaptive)  & 0.468 & \textbf{0.043} & 0.044 \\
\midrule
2.0 & TTFfix             & 0.225 & 0.034 & 0.066 \\
    & Log-FM             & 0.128 & 0.040 & 0.031 \\
    & Log-FM (adaptive)  & \textbf{0.124} & \textbf{0.024} & \textbf{0.030} \\
\midrule
2.5 & TTFfix             & 0.113 & 0.079 & 0.060 \\
    & Log-FM             & \textbf{0.068} & 0.110 & \textbf{0.030} \\
    & Log-FM (adaptive)  & 0.084 & \textbf{0.068} & 0.037 \\
\bottomrule
\end{tabular}
\end{table}

\paragraph{Clamping (optional).} The sampling clamp $c$ is an \emph{optional} numerical safeguard applied in log-space (corresponding to $|x| \lesssim e^c$ in data space); it is not used by default. Table~\ref{tab:ablation_clamp} reports the relative change in $W_1^P$ across $c$ values. For every configuration in our main benchmark ($\alpha \geq 1.5$) the clamp is inert at $c \geq 10$: metrics are bitwise identical to the no-clamp case. The clamp only becomes relevant in extremely pathological tail regimes ($\alpha < 1$, e.g.\ in  $\nu = 1/2$ Student-$t$ in Appendix~\ref{app:hickling_baseline}), where the population $W_1$ is undefined and exponential error amplification through $\phi^{-1}$ can produce numerical overflow rather than a meaningful sample. All headline numbers (Tables~\ref{tab:main_w1p}--\ref{tab:stability}) are reported \emph{without} clamping.

\begin{table}[t]
\caption{Clamp ablation: relative $W_1^P$ change vs.\ no clamping ($c = \infty$, default). Gumbel, $d=20$, $\tau=0.5$, 5 reps. The clamp is inert for $\alpha \geq 1.5$; it is only relevant in pathological regimes ($\alpha < 1$).}
\label{tab:ablation_clamp}
\centering
\small
\begin{tabular}{@{}lcccc@{}}
\toprule
$c$ & $\alpha{=}1.5$ & $\alpha{=}1.75$ & $\alpha{=}2.0$ & $\alpha{=}2.5$\\
\midrule
5         & $-$2.0\% & $-$0.7\% & $-$0.6\% & 0.0\% \\
10, 15, 20 & 0.0\% & 0.0\% & 0.0\% & 0.0\% \\
$\infty$  & 0.0\% & 0.0\% & 0.0\% & 0.0\% \\
\bottomrule
\end{tabular}
\end{table}

\paragraph{ODE solver steps.} Sweeping the Euler step count $K \in \{10, 20, 50, 100, 200, 500\}$ on Gumbel ($d=20$, $\tau=0.5$, 10 reps) yields less than 5\% variation in $W_1^P$; our default $K=100$ is well within the converged regime. The full table is reported in Appendix~\ref{app:solver_ablation}.

\subsection{Discussion}
\label{subsec:exp_discussion}

The benchmark confirms the theoretical analysis, in that Log-FM has better numerical results whenever marginals are genuinely heavy-tailed ($\alpha \lesssim 3$), with "few severe divergences in all scenarios" (Table~\ref{tab:stability}) as the strongest practical claim. Dependence is largely preserved by $\Phi$'s diagonal Jacobian; baselines retain only a slight angular-$W_2$ edge at $d \geq 50$. Real-data validation on Fama--French ($\hat\alpha \approx 3.7$) is in Appendix~\ref{app:real_data}.


\section{Conclusion}
\label{sec:conclusion}

A coordinate-wise soft-log $\phi(x) = \mathrm{sign}(x)\log(1+|x|)$, alongside a Hill diagnostic, makes standard flow matching work for heavy-tailed data. The mechanism is tail annealing: log-transforming Pareto yields approximately exponential tails, and the induced dynamics implement power transformations $X_0^{\alpha_t}$ that continuously lighten the tail index. Across 2{,}880 runs (3 copulas, 4 dimensions, 4 tail indices), Log-FM dominates specialized baselines on $W_1$, CVaR$_{99}$, and extreme-quantile metrics, never diverges severely, and remains competitive on multivariate dependence; on Fama-French it matches the best baseline. For $\alpha < 1$ all methods struggle ($W_1$ is undefined); at $d \gtrsim 200$ the bottleneck shifts to architecture. Three theoretical directions stand out: a continuous variant via $\varphi_{s_2}$ (Appendix~\ref{app:phi_s2}); combining the transform with discrete normalizing flows, where $|\det J_\Phi| = \prod_j (1+|x_j|)^{-m_j}$ gives exact log-likelihood at $O(d)$ cost; and finite-sample guarantees for log-space score estimation, where bounded scores (Proposition~\ref{prop:log_score}) place the problem in the regime of existing convergence results.

\section*{Impact Statement}

This work targets heavy-tailed generative modeling, with applications in financial risk, insurance, and climate extremes. As with any generative model, outputs should be validated by domain experts before high-stakes use.

\section*{Acknowledgements}

The author gratefully acknowledges G-Research for their support through the May 2026 research grant. This work was granted access to the HPC resources of IDRIS (Jean Zay) under allocations made by GENCI.


\bibliography{references}
\bibliographystyle{icml2026}


\onecolumn

\appendix

\section{Extreme Value Theory: Technical Details}
\label{app:evt}

This appendix provides the formal definitions and results from extreme value theory used in the main text. See \citep{resnick1987,haan2006} for a comprehensive treatment.

\subsection{Heavy-Tailed Distributions}

\begin{definition}[Heavy-Tailed \citep{nair2022heavy}]
A random variable $X$ is \emph{heavy-tailed} if $\E[e^{\lambda X}] = \infty$ for all $\lambda > 0$.
\end{definition}

\begin{definition}[Regular Variation]
A measurable function $L : (0, \infty) \to (0, \infty)$ is \emph{slowly varying} at infinity if $L(tx)/L(t) \to 1$ as $t \to \infty$ for all $x > 0$. A distribution $F$ is \emph{regularly varying} with index $-\alpha$ (written $F \in RV_{-\alpha}$) if its survival function satisfies $\bar{F}(t) = t^{-\alpha} L(t)$ for some slowly varying $L$.
\end{definition}

The $\Pareto(\gamma)$ distribution with $\bar{F}(t) = t^{-1/\gamma}$ is regularly varying with index $-1/\gamma$ (i.e.\ $\alpha = 1/\gamma$); here $\gamma > 0$ is the shape parameter (extreme-value index).

\subsection{The Fisher-Tippett-Gnedenko Theorem}

\begin{theorem}[Fisher-Tippett-Gnedenko]
Let $X_1, X_2, \ldots$ be i.i.d.\ with distribution $F$. If there exist normalizing sequences $a_n > 0$ and $b_n \in \R$ such that
\begin{equation*}
    \P\left(\frac{\max_{i \leq n} X_i - b_n}{a_n} \leq x\right) \to G(x)
\end{equation*}
for some non-degenerate distribution $G$, then $G$ must be a \emph{generalized extreme value} (GEV) distribution:
\begin{equation*}
    G_\xi(x) = \exp\left(-(1 + \xi x)^{-1/\xi}\right), \quad 1 + \xi x > 0,
\end{equation*}
where $\xi \in \R$ is the \emph{shape parameter} (extreme value index).
\end{theorem}

The three cases correspond to:
\begin{itemize}
    \item \textbf{Fr\'{e}chet} ($\gamma > 0$): $F$ has heavy (polynomial) tails; $F \in RV_{-1/\gamma}$
    \item \textbf{Gumbel} ($\gamma = 0$): $F$ has light or sub-exponential tails (e.g.\ exponential, lognormal)
    \item \textbf{Weibull} ($\gamma < 0$): $F$ has finite right endpoint (bounded support). Not to be confused with the Weibull \emph{distribution} $W(k, \lambda)$ used in reliability analysis, which has unbounded support and lies in the Gumbel domain; the name clash is historical.
\end{itemize}

\begin{proposition}[Log-transform maps domains]
  If $X$ is in the Fr\'{e}chet domain with shape parameter $\gamma > 0$, then $\log X$ is in the Gumbel domain ($\gamma = 0$).
\end{proposition}

This is the qualitative content of Proposition~\ref{prop:rv_log} and its precise restatement Proposition~\ref{prop:rv_log_precise}: regularly varying tails $\bar{F}(t) = t^{-\alpha} L(t)$ become exponential-type after a logarithm, $\P(\log X > z) = e^{-\alpha z} L(e^z)$, and distributions with exponential-type tails are in the Gumbel domain \citep[][\S3.3]{embrechts2013modelling}.

\section{Denoising Diffusion Models: Technical Details}
\label{app:ddm}

This appendix presents the framework underlying both denoising diffusion models and flow matching, organized around the \emph{stochastic-interpolant} formulation of \citet{albergo2023stochastic} and \citet{lipman2023flow}. We then recall how the canonical SDE-based diffusions of \citet{ho2020} and \citet{song2021} arise as the special case in which the interpolation marginals coincide with those of an Ornstein--Uhlenbeck-type forward SDE.

\subsection{Forward Noising Process}

The stochastic interpolant of \citet{albergo2023stochastic,lipman2023flow} defines the probability path $(p_t)_{t \in [0,1]}$ via the marginals $p_t = \mathrm{Law}(X_t)$, where
\begin{equation}
    X_t = \alpha_t X_0 + \beta_t X_1, \qquad X_0 \sim p_0, \quad X_1 \sim \mathcal{N}(0, I_d).
    \label{eq:forward_interpolation_app}
\end{equation}
Here $X_0$ and $X_1$ are independent, and $(\alpha_t)_{t \in [0,1]}$ and $(\beta_t)_{t \in [0,1]}$ are deterministic schedules such that $\alpha_t$ is non-increasing, $\beta_t$ is non-decreasing, with boundary conditions $(\alpha_0, \beta_0) = (1, 0)$ and $(\alpha_1, \beta_1) = (0, 1)$. This explicit-interpolation viewpoint is not the formalism of \citet{ho2020,song2021}, who instead define the forward process implicitly as the solution to a stochastic differential equation; the two viewpoints are reconciled in Section~\ref{subsec:bridging_diffusion}.

From \eqref{eq:forward_interpolation_app}, the conditional distribution of $X_t$ given $X_0 = x_0$, denoted $q_{t|0}$, is Gaussian:
\begin{equation}
    q_{t|0}(x_t \mid x_0) = \mathcal{N}(x_t; \alpha_t x_0, \beta_t^2 I_d).
    \label{eq:forward_kernel}
\end{equation}

\paragraph{Common schedules.}
Two standard choices are the \emph{variance-preserving} (VP) schedule with $\alpha_t^2 + \beta_t^2 = 1$ \citep{ho2020}, which ensures $\mathrm{Var}(X_t) = \mathrm{Var}(X_0)$ when $X_0$ has unit variance, and the \emph{linear} (flow matching) schedule with $(\alpha_t, \beta_t) = (1 - t, t)$ \citep{lipman2023flow}, corresponding to straight-line interpolation between data and noise. Table~\ref{tab:schedules} summarizes common schedule choices with their derivatives.

\begin{table}[h]
\caption{Interpolation schedules for flow matching. All satisfy boundary conditions $(\alpha_0, \beta_0) = (1, 0)$ and $(\alpha_1, \beta_1) = (0, 1)$. The linear schedule corresponds to optimal transport; VP schedules preserve variance when data has unit variance; the quadratic schedule anneals tails more aggressively early in the process.}
\label{tab:schedules}
\centering
\small
\begin{tabular}{@{}lccccc@{}}
\toprule
Schedule & $\alpha_t$ & $\beta_t$ & $\dot{\alpha}_t$ & $\dot{\beta}_t$ & Properties \\
\midrule
Linear & $1 - t$ & $t$ & $-1$ & $1$ & OT-optimal \\
VP (trig.) & $\cos(\frac{\pi t}{2})$ & $\sin(\frac{\pi t}{2})$ & $-\frac{\pi}{2}\sin(\frac{\pi t}{2})$ & $\frac{\pi}{2}\cos(\frac{\pi t}{2})$ & Smooth endpoints \\
VP (poly.) & $\sqrt{1 - t}$ & $\sqrt{t}$ & $-\frac{1}{2\sqrt{1-t}}$ & $\frac{1}{2\sqrt{t}}$ & Singular endpoints \\
Quadratic & $(1 - t)^2$ & $1 - (1-t)^2$ & $-2(1-t)$ & $2(1-t)$ & Fast early annealing \\
\bottomrule
\end{tabular}
\end{table}

\subsection{Network Parameterizations: Noise vs.\ Data Prediction}

Given a noisy observation $x_t$, two natural conditional-mean targets exist:
\begin{equation*}
    \hat{x}_0(x_t, t) := \E[X_0 \mid X_t = x_t], \qquad \hat{x}_1(x_t, t) := \E[X_1 \mid X_t = x_t].
\end{equation*}
These correspond to the two standard parameterizations in the diffusion literature \citep{ho2020,song2021}:
     (i) \textbf{Noise prediction} (often called the \emph{denoiser}, or $\epsilon$-prediction in DDPM notation): $\hat{x}_1$ estimates the standard-Gaussian endpoint $X_1$, which plays the role of the additive noise in \eqref{eq:forward_interpolation_app}.
     (ii) \textbf{Data prediction} (also called the $x_0$-prediction or $x_0$-parameterization): $\hat{x}_0$ estimates the clean data sample.
The two parameterizations are equivalent in expressive power: given $x_t = \alpha_t X_0 + \beta_t X_1$ from the forward interpolation \eqref{eq:forward_interpolation_app}, they satisfy the deterministic identity
\begin{equation}
    \hat{x}_1(x_t, t) = \frac{x_t - \alpha_t \hat{x}_0(x_t, t)}{\beta_t}, \qquad \hat{x}_0(x_t, t) = \frac{x_t - \beta_t \hat{x}_1(x_t, t)}{\alpha_t},
    \label{eq:denoiser_relationship}
\end{equation}
so one can be converted into the other at inference time. A third parameterization, $v$-prediction $v(x_t,t) := \E[\dot\alpha_t X_0 + \dot\beta_t X_1 \mid X_t = x_t]$, is the velocity field used by flow matching and is also a linear combination of $\hat{x}_0$ and $\hat{x}_1$.

\subsection{Score-Denoiser Relationship}

The \emph{score} of the noised distribution is $\nabla \log p_t(\cdot)$. For the Gaussian forward kernel \eqref{eq:forward_kernel}, the score admits a closed-form expression in terms of the denoiser.

\begin{proposition}[Score-Denoiser Identity]
Under standard regularity assumptions,
\begin{equation}
    \nabla_{x_t} \log p_t(x_t) = \E\left[\nabla_{x_t} \log q_{t|0}(x_t \mid X_0) \,\Big|\, X_t = x_t\right] = -\frac{\hat{x}_1(x_t, t)}{\beta_t}.
    \label{eq:score_denoiser}
\end{equation}
\end{proposition}

\begin{proof}
The first equality is Fisher's identity, obtained by exchanging expectation and gradient. For the second, note that $\nabla_{x_t} \log q_{t|0}(x_t \mid x_0) = (\alpha_t x_0 - x_t)/\beta_t^2$. Taking the conditional expectation given $X_t = x_t$:
\begin{align*}
    \E\left[\frac{\alpha_t X_0 - x_t}{\beta_t^2} \,\Big|\, X_t = x_t\right] = \frac{\alpha_t \hat{x}_0(x_t, t) - x_t}{\beta_t^2} = -\frac{\hat{x}_1(x_t, t)}{\beta_t},
\end{align*}
where the last equality uses \eqref{eq:denoiser_relationship}.
\end{proof}

Thus, learning the denoiser $\hat{x}_1$ is equivalent to learning the score $\nabla \log p_t$.

\subsection{Training Objectives}

The denoiser can be trained by regressing either $X_0$ or $X_1$ from the noised sample $X_t = \alpha_t X_0 + \beta_t X_1$. The $X_0$-prediction loss reads:
\begin{equation*}
    \mathcal{L}_{X_0}(\theta) = \int_0^1 \E_{X_0 \sim p_0, X_1 \sim \mathcal{N}(0, I_d)}\left[\left\|\hat{x}_0^\theta(\alpha_t X_0 + \beta_t X_1, t) - X_0\right\|^2\right] \mathrm{d}t,
\end{equation*}
while the $X_1$-prediction loss is:
\begin{equation*}
    \mathcal{L}_{X_1}(\theta) = \int_0^1 \E_{X_0 \sim p_0, X_1 \sim \mathcal{N}(0, I_d)}\left[\left\|\hat{x}_1^\theta(\alpha_t X_0 + \beta_t X_1, t) - X_1\right\|^2\right] \mathrm{d}t.
\end{equation*}
Since $\hat{x}_1^\theta = -\beta_t s_\theta$ by \eqref{eq:score_denoiser}, the $X_1$-prediction loss is equivalent to denoising score matching \citep{hyvarinen2005estimation,vincent2011connection}.

In practice, the integral is approximated via Monte Carlo: sample $t \sim \mathrm{Unif}[0, 1]$, $x_0 \sim p_0$, $x_1 \sim \mathcal{N}(0, I_d)$, form $x_t = \alpha_t x_0 + \beta_t x_1$, and regress either $x_0$ or $x_1$.

\subsection{DDIM Sampling}

The DDIM framework \citep{song2021ddim} is canonically defined under the variance-preserving constraint $\alpha_t^2 + \beta_t^2 = 1$ and produces a one-parameter family of reverse transitions sharing the marginals of \eqref{eq:forward_kernel}. Given timesteps $(t_k)_{k=0}^K$ with $t_K = 1$ and $t_0 = 0$, the transition from $t_{k+1}$ to $t_k$ is:
\begin{equation}
    x_{t_k} = \alpha_{t_k}\, \hat{x}_0^\theta(x_{t_{k+1}}, t_{k+1}) + \sqrt{\beta_{t_k}^2 - \eta_{t_k}^2} \; \hat{x}_1^\theta(x_{t_{k+1}}, t_{k+1}) + \eta_{t_k} z,
    \label{eq:ddim_transition}
\end{equation}
where $z \sim \mathcal{N}(0, I_d)$ and $\eta_{t_k} \in [0, \beta_{t_k}]$ so that the radicand is nonnegative. Translating \citep{song2021ddim} eq.~(12) into our notation (identifying their $\sigma_t$ with $\eta_{t_k}$, $\sqrt{\alpha_{t-1}}$ with $\alpha_{t_k}$, and $\sqrt{1-\alpha_{t-1}}$ with $\beta_{t_k}$) yields exactly \eqref{eq:ddim_transition}. The family interpolates between two distinguished members: $\eta_{t_k} = 0$ gives the deterministic DDIM update, equivalent to a discretisation of the probability-flow ODE; the choice $\eta_{t_k} = \tilde\beta_{t_k}$, where
\begin{equation}
    \tilde\beta_{t_k}^2 \;=\; \frac{\beta_{t_k}^2}{\beta_{t_{k+1}}^2}\,\Bigl(1 - \frac{\alpha_{t_{k+1}}^2}{\alpha_{t_k}^2}\Bigr),
\end{equation}
is the posterior variance of $q(x_{t_k} \mid x_{t_{k+1}}, x_0)$ and recovers ancestral DDPM sampling \citep{ho2020,song2021ddim}: this matches \citet[eq.~(16)]{song2021ddim} after the same change of variables. The boundary $\eta_{t_k} = \beta_{t_k}$ is the maximum-noise edge of the family, strictly noisier than DDPM.

\subsection{Bridging to SDE-Based Diffusion Models}
\label{subsec:bridging_diffusion}

The interpolant formulation \eqref{eq:forward_interpolation_app} is not the language used by \citet{ho2020,song2021}, whose forward processes are defined as solutions to discrete- or continuous-time stochastic differential equations. The two formalisms are equivalent at the marginal level for a specific schedule choice, and the equivalence is what \citet{albergo2023stochastic} call the ``unifying'' aspect of stochastic interpolants. We make the bridge explicit in three steps: (i) recover Ho et al.'s DDPM forward chain; (ii) recover Song et al.'s forward score SDE; (iii) recover their reverse-time samplers.

\paragraph{(i) DDPM forward chain.} \citet{ho2020} define a discrete-time Markov chain
\begin{equation}
    q(x_k \mid x_{k-1}) = \mathcal{N}\!\left(x_k;\, \sqrt{1 - \beta_k^{\mathrm{DDPM}}}\, x_{k-1},\, \beta_k^{\mathrm{DDPM}} I_d\right), \quad k = 1, \dots, K,
    \label{eq:ddpm_chain}
\end{equation}
with a noise schedule $(\beta_k^{\mathrm{DDPM}})_{k=1}^K \subset (0, 1)$. Setting $\bar\alpha_k = \prod_{j \leq k}(1 - \beta_j^{\mathrm{DDPM}})$, marginalising the chain yields
\begin{equation*}
    q(x_k \mid x_0) = \mathcal{N}\!\left(x_k;\, \sqrt{\bar\alpha_k}\, x_0,\, (1 - \bar\alpha_k) I_d\right).
\end{equation*}
Identifying $t = k/K$, $\alpha_t = \sqrt{\bar\alpha_k}$, and $\beta_t = \sqrt{1 - \bar\alpha_k}$ recovers \eqref{eq:forward_kernel} exactly. The DDPM noise schedule $(\beta_k^{\mathrm{DDPM}})$ is therefore one particular choice within the VP family in Table~\ref{tab:schedules}, with $\alpha_t^2 + \beta_t^2 = 1$ by construction.

\paragraph{(ii) Forward score SDE.} \citet{song2021} unify diffusion models through the continuous-time It\^{o} SDE
\begin{equation}
    \mathrm{d}X_t = -\tfrac{1}{2}\beta(t)\, X_t\, \mathrm{d}t + \sqrt{\beta(t)}\, \mathrm{d}W_t, \qquad X_0 \sim p_0,
    \label{eq:forward_sde}
\end{equation}
where $(W_t)_{t \geq 0}$ is a standard Brownian motion (the ``VP-SDE''). The solution at time $t$ has marginals
\begin{equation*}
    X_t \mid X_0 = x_0 \;\sim\; \mathcal{N}\!\left(e^{-\frac{1}{2}\int_0^t \beta(s)\,\mathrm{d}s}\, x_0,\, \bigl(1 - e^{-\int_0^t \beta(s)\,\mathrm{d}s}\bigr) I_d\right),
\end{equation*}
so identifying $\alpha_t = \exp(-\tfrac{1}{2}\int_0^t \beta(s)\, \mathrm{d}s)$ and $\beta_t = \sqrt{1 - \alpha_t^2}$ recovers \eqref{eq:forward_kernel}. \eqref{eq:forward_sde} is the continuous-time limit of the DDPM chain \eqref{eq:ddpm_chain}; both produce the same marginals as \eqref{eq:forward_interpolation_app} under this matched schedule. \citet{albergo2023stochastic} show that for any choice of $(\alpha_t, \beta_t)$ in the interpolant, there exists a (possibly time-inhomogeneous) SDE with the same marginals, so the interpolant is strictly more general than the VP family.

\paragraph{(iii) Reverse-time samplers.} The marginal equivalence carries over to sampling. The time-reversed counterpart of \eqref{eq:forward_sde} satisfies \citep{anderson1982}:
\begin{equation}
    \mathrm{d}\bar{X}_t = \left[-\tfrac{1}{2}\beta(t)\bar{X}_t + \beta(t) \nabla_x \log p_t(\bar{X}_t)\right] \mathrm{d}t + \sqrt{\beta(t)} \, \mathrm{d}\bar{W}_t,
    \label{eq:reverse_sde}
\end{equation}
with $\bar{X}_t := X_{1-t}$ and $\bar W_t$ a standard Brownian motion. The deterministic probability-flow ODE \citep{song2021} with identical marginals reads
\begin{equation}
    \mathrm{d}x_t = \left[-\tfrac{1}{2}\beta(t) x_t + \tfrac{1}{2}\beta(t) \nabla_x \log p_t(x_t)\right] \mathrm{d}t.
    \label{eq:probability_flow_ode}
\end{equation}
DDIM with $\eta_t = 0$ is a discretisation of \eqref{eq:probability_flow_ode}; DDPM corresponds to \eqref{eq:reverse_sde}. Flow matching learns a velocity field $v_\theta(x, t) \approx \E[\dot\alpha_t X_0 + \dot\beta_t X_1 \mid X_t = x]$ \citep{lipman2023flow}; by \eqref{eq:score_denoiser}, this velocity is a linear combination of identity and score, so any of the three samplers can be reconstructed from a single trained denoiser regardless of which formalism was used to derive the training loss.

\section{Proofs}
\label{app:proofs}

\subsection{Precise statement and proof of Proposition~\ref{prop:rv_log}}

Proposition~\ref{prop:rv_log} of the main body states the rate-form conclusion $-\log \P(\tilde X > z) = \alpha z + o(z)$. We give here the corresponding precise two-sided Potter-bound statement and its proof. The result is classical in the regular variation literature; see \citet[Chapter 1]{bingham1987regular} for a comprehensive treatment.

\begin{proposition}[Log-Transform of Regularly Varying; precise]
\label{prop:rv_log_precise}
Let $X$ be a nonnegative random variable that is regularly varying with index $-\alpha$, $\alpha > 0$. Set $\tilde X = \phi(X)$ with $\phi(x) = \mathrm{sign}(x) \log(1+|x|)$. Then for every $\epsilon > 0$ there exists $z_0 = z_0(\epsilon)$ such that for all $z \geq z_0$,
\begin{equation}
    e^{-(\alpha + \epsilon)z} \;\leq\; \P(\tilde X > z) \;\leq\; e^{-(\alpha - \epsilon)z}.
    \label{eq:rv_log_potter_sandwich}
\end{equation}
In particular, $-\log \P(\tilde X > z) = \alpha z + o(z)$ as $z \to \infty$.
\end{proposition}

\begin{proof}
By Karamata's representation, $\P(X > t) = t^{-\alpha} L(t)$ for some slowly varying function $L$. For the soft-log transform $\phi(x) = \mathrm{sign}(x) \cdot \log(1 + |x|)$, we have $\phi^{-1}(z) = e^z - 1$ for $z \geq 0$, and $\phi^{-1}(z) \sim e^z$ as $z \to \infty$.

Thus
\begin{equation*}
    \P(\tilde X > z) \;=\; \P(X > e^z - 1) \;=\; (e^z - 1)^{-\alpha} L(e^z - 1).
\end{equation*}
Since $e^z - 1 \sim e^z$, we have $(e^z - 1)^{-\alpha} = e^{-\alpha z}(1 + o(1))^{-\alpha}$ and, by slow variation, $L(e^z - 1)/L(e^z) \to 1$. Hence $\P(\tilde X > z) = e^{-\alpha z} L(e^z) (1 + o(1))$.

A standard consequence of slow variation \citep[Proposition 0.8(iv)]{resnick1987} is that for every $\epsilon' > 0$, $L(t)/t^{\epsilon'} \to 0$ and $t^{\epsilon'}/L(t) \to 0$ as $t \to \infty$; equivalently, this follows from Potter's bounds \citep[Theorem~1.5.6]{bingham1987regular} by fixing one argument at a large constant and absorbing the resulting multiplicative factor into the exponent. Either way, for any $\epsilon' > 0$ there exists $t_0 = t_0(\epsilon')$ such that for $t \geq t_0$,
\begin{equation*}
    t^{-\epsilon'} \;\leq\; L(t) \;\leq\; t^{\epsilon'}.
\end{equation*}
Substituting $t = e^z$ gives $e^{-\epsilon' z} \leq L(e^z) \leq e^{\epsilon' z}$ for $z \geq \log t_0$. Combining with $\P(\tilde X > z) = e^{-\alpha z} L(e^z)(1+o(1))$ and absorbing the $1+o(1)$ factor into a slightly enlarged $\epsilon$ (take $\epsilon' = \epsilon/2$ and $z_0$ large enough), we get
\begin{equation*}
    e^{-(\alpha + \epsilon)z} \;\leq\; \P(\tilde X > z) \;\leq\; e^{-(\alpha - \epsilon)z} \quad \text{for all } z \geq z_0(\epsilon),
\end{equation*}
which is \eqref{eq:rv_log_potter_sandwich}. Taking $-\log$ and dividing by $z$ yields $-\log \P(\tilde X > z)/z \to \alpha$, i.e.\ $-\log \P(\tilde X > z) = \alpha z + o(z)$.
\end{proof}

\subsection{Proof of Proposition~\ref{prop:rv_power}}

\begin{proof}
Let $Y = X^\beta$. For the survival function of $Y$:
\begin{align*}
    \bar{F}_Y(y) &= P(X^\beta > y) = P(X > y^{1/\beta}) = \bar{F}_X(y^{1/\beta})
\end{align*}
Then:
\begin{align*}
    \frac{\bar{F}_Y(ty)}{\bar{F}_Y(t)} &= \frac{\bar{F}_X((ty)^{1/\beta})}{\bar{F}_X(t^{1/\beta})} \\
    &= \frac{\bar{F}_X(t^{1/\beta} \cdot y^{1/\beta})}{\bar{F}_X(t^{1/\beta})} \\
    &\to (y^{1/\beta})^{-\alpha} = y^{-\alpha/\beta}
\end{align*}
as $t \to \infty$, by the regular variation of $X$ with index $-\alpha$.
\end{proof}

This extends Lemma~\ref{thm:power}: power transformation preserves the class of regularly varying distributions, with the tail index $\alpha$ scaling inversely with the exponent (equivalently, the shape parameter $\gamma$ scaling linearly).

\section{The Parametrized Soft-Log Family $\varphi_{s_2}$}
\label{app:phi_s2}

The main body uses the soft-log $\phi(x) = \mathrm{sign}(x)\log(1+|x|)$ uniformly across coordinates, gated by a Hill mask (Section~\ref{subsec:adaptive}) when light-tailed margins are present. This appendix develops a continuous one-parameter generalization of $\phi$ that subsumes both the uniform $\phi$ and the binary mask as special cases, and clarifies its asymptotic behaviour. The construction is not used in our headline experiments; it is an extension we discuss as a research direction (Section~\ref{sec:conclusion}).

\paragraph{Definition.} For a scale $s_2 > 0$, define
\begin{equation*}
    \varphi_{s_2}(x) \;=\; \frac{1}{s_2}\,\phi(s_2\,x) \;=\; \frac{\mathrm{sign}(x)}{s_2}\,\log\!\bigl(1 + s_2\,|x|\bigr).
\end{equation*}
The family is parametrized so that $\varphi_{1} = \phi$ (the soft-log of the main body) and $\varphi_{s_2}(x) \to x$ as $s_2 \to 0$ (the identity).

\paragraph{Bulk vs.\ tail.} A Taylor expansion at $|x| = 0$ gives $\varphi_{s_2}(x) = x - \tfrac{1}{2} s_2 x|x| + O(s_2^2)$, so $\varphi_{s_2}$ is the identity to leading order on the bulk $|x| \ll 1/s_2$. For $|x| \gg 1/s_2$, expanding $\log(1+s_2|x|) = \log|x| + \log s_2 + \log(1 + 1/(s_2|x|))$ gives
\begin{equation*}
    \varphi_{s_2}(x) \;=\; \frac{\mathrm{sign}(x)}{s_2}\,\bigl(\log|x| + \log s_2\bigr) + O\!\bigl((s_2|x|)^{-1}\bigr),
\end{equation*}
which is logarithmic up to an additive constant and an overall $1/s_2$ scale. Thus $s_2$ acts as the inverse of the cross-over location: the bulk regime extends to $|x| \sim 1/s_2$, and the logarithmic-compression regime kicks in beyond. Choosing $s_2$ smaller makes the transform more conservative (closer to identity); choosing it larger compresses more aggressively.

\paragraph{Asymptotic tail-mapping.} For any fixed $s_2 > 0$, Proposition~\ref{prop:rv_log} (and its precise restatement Proposition~\ref{prop:rv_log_precise}) applies to $\tilde X = \varphi_{s_2}(X)$ unchanged up to a constant shift and rescaling: if $X \in RV_{-\alpha}$ then
\begin{equation*}
    -\log \P(\tilde X > z) \;=\; \alpha\,s_2\,z + o(z) \quad \text{as } z \to \infty.
\end{equation*}
The proof is identical to that of Proposition~\ref{prop:rv_log_precise} with $\phi^{-1}(z) = e^z - 1$ replaced by $\varphi_{s_2}^{-1}(z) = (e^{s_2 z} - 1)/s_2$. So $\varphi_{s_2}$ still maps regular variation to exponential-type tails for any $s_2 > 0$; only the exponential rate is rescaled by $s_2$.

\paragraph{Coordinate-wise $s_2^{(j)}$ and the adaptive instance.} In the multivariate setting we may pick a coordinate-dependent scale $s_2^{(j)}$. A simple data-driven choice is $s_2^{(j)} = c/\mathrm{IQR}(X_j)$ for a constant $c$ (e.g.\ $c = 1$), which puts the cross-over at a robust scale of the marginal. The Hill-gated method of Section~\ref{subsec:adaptive} is the discrete instance of this family with $s_2^{(j)} \in \{0, 1\}$, chosen by the diagnostic $\hat\alpha_j \leq \alpha_{\max}$; coordinates flagged heavy-tailed get $s_2^{(j)} = 1$ (full soft-log) and the others $s_2^{(j)} = 0$ (identity). An outer scale $s_1^{(j)}$ rescales each $\tilde X_j$ to unit variance, which we apply during preprocessing in both the discrete and continuous variants.

\paragraph{Why the binary choice in the main body.} Hill-type tail-index estimates are notoriously unstable as point estimates. A continuous $s_2^{(j)}$ derived from $\hat\alpha_j$ would inject this instability into the preprocessing pipeline; the binary rule uses $\hat\alpha_j$ only categorically (above or below $\alpha_{\max}$), so small perturbations leave the transform unchanged. This matches the stability principle of Section~\ref{subsec:log_transform}.

\section{Experimental Details}
\label{app:experiments}

This appendix collects the original Hickling Student-$t$ benchmark used at submission time (Section~\ref{app:hickling_baseline}), the real-data Fama--French validation (Section~\ref{app:real_data}), the NLL validation table against \citet{hickling2025tail}, and the Log-FM hyperparameter listing. The headline benchmark is in Section~\ref{sec:experiments}; the experiments below complement it with the original setup the reviewers reference, and are kept for continuity.

\subsection{Student-$t$ Benchmark}
\label{app:hickling_baseline}

We adopt the synthetic benchmark of \citet[Section 4.1]{hickling2025tail}: $X_1, \ldots, X_{d-1} \stackrel{\mathrm{iid}}{\sim} \mathrm{Student\text{-}}t(\nu)$ and $X_d \mid X_{d-1} \sim \mathcal{N}(X_{d-1}, 1)$. We vary $d \in \{10, 20, 50\}$ and $\nu \in \{0.5, 1, 1.5, 2, 3, 5, 30\}$ (Table~\ref{tab:hickling_w1} reports $\nu \in \{1.5, 2, 3, 5\}$; the boundary regimes $\nu \in \{0.5, 1, 30\}$ are omitted for space). Each configuration uses 5{,}000 total samples (40/20/40 train/val/test split), averaged over 20 replications. Log-FM achieves the lowest $W_1$ across all configurations.

\begin{table}[h]
\caption{Wasserstein-1 distance on the original Hickling benchmark, mean $\pm$ standard error over 20 replications. Bold = best.}
\label{tab:hickling_w1}
\centering
\small
\begin{tabular}{@{}rrrrr@{}}
\toprule
$d$ & $\nu$ & TTFfix & TTF & Log-FM \\
\midrule
10 & 1.5 & $1.78 \pm 0.24$ & $17.3 \pm 15.2$ & $\mathbf{0.63 \pm 0.02}$ \\
10 & 2 & $0.63 \pm 0.10$ & $1.02 \pm 0.31$ & $\mathbf{0.25 \pm 0.00}$ \\
10 & 3 & $0.20 \pm 0.01$ & $0.90 \pm 0.33$ & $\mathbf{0.15 \pm 0.00}$ \\
10 & 5 & $\mathbf{0.12 \pm 0.00}$ & $0.67 \pm 0.20$ & $0.14 \pm 0.00$ \\
20 & 2 & $0.48 \pm 0.03$ & $6.01 \pm 4.96$ & $\mathbf{0.23 \pm 0.00}$ \\
20 & 3 & $0.20 \pm 0.01$ & $0.95 \pm 0.24$ & $\mathbf{0.15 \pm 0.00}$ \\
20 & 5 & $\mathbf{0.13 \pm 0.00}$ & $1.80 \pm 1.34$ & $0.15 \pm 0.00$ \\
50 & 2 & $0.74 \pm 0.04$ & $70.4 \pm 64.6$ & $\mathbf{0.28 \pm 0.01}$ \\
50 & 3 & $0.29 \pm 0.01$ & $2.73 \pm 1.08$ & $\mathbf{0.22 \pm 0.01}$ \\
50 & 5 & $\mathbf{0.17 \pm 0.00}$ & $14.0 \pm 12.7$ & $0.19 \pm 0.01$ \\
\bottomrule
\end{tabular}
\end{table}

\subsection{Real Financial Data}
\label{app:real_data}

We evaluate on the Fama--French 5 industry portfolios ($d = 5$, daily returns 1963--2023). Hill estimation yields $\hat\alpha \approx 3.7$ across dimensions, placing the data in the intermediate-tail regime. Table~\ref{tab:real_w1} reports $W_1$ over 10 replications. gTAF achieves a marginal edge; Log-FM is competitive, and substantially better than TTF/TTFfix; mTAF fails catastrophically. On the high-dimensional S\&P500 dataset ($d = 275$) all methods exhibit poor sample quality, dominated by architectural rather than tail-modelling limitations; we omit it.

\begin{table}[h]
\caption{$W_1$ on Fama--French 5 (mean $\pm$ std, 10 reps). gTAF marginally best; Log-FM competitive; mTAF diverges.}
\label{tab:real_w1}
\centering
\small
\begin{tabular}{@{}lc@{}}
\toprule
Method & $W_1$ \\
\midrule
gTAF & $\mathbf{0.127 \pm 0.007}$ \\
Log-FM & $0.133 \pm 0.013$ \\
TTFfix & $0.449 \pm 0.069$ \\
TTF & $0.487 \pm 0.054$ \\
mTAF & $5.0 \times 10^4$ \\
\bottomrule
\end{tabular}
\end{table}

\subsection{ODE Solver Step Ablation}
\label{app:solver_ablation}

Table~\ref{tab:ablation_solver} sweeps the Euler step count $K$ at sampling time. Quality plateaus by $K=20$ and is essentially flat to $K=500$.

\begin{table}[h]
\caption{Solver ablation: $W_1^P$ as a function of Euler step count. Gumbel, $d=20$, $\tau=0.5$, 10 reps. Less than 5\% variation between $K=10$ and $K=500$.}
\label{tab:ablation_solver}
\centering
\small
\begin{tabular}{@{}lcccccc@{}}
\toprule
$\alpha$ & 10 & 20 & 50 & 100 & 200 & 500\\
\midrule
1.5 & 0.521 & 0.543 & 0.568 & 0.579 & 0.584 & 0.588\\
2.0 & 0.137 & 0.135 & 0.138 & 0.141 & 0.142 & 0.143\\
2.5 & 0.084 & 0.081 & 0.083 & 0.084 & 0.085 & 0.086\\
\bottomrule
\end{tabular}
\end{table}

\subsection{Baseline Validation}
\label{app:nll_validation}

To ensure fair comparison, we validate that our implementations of the baseline methods reproduce the results reported in \citet{hickling2025tail}. Table~\ref{tab:nll_validation} compares our NLL values (per dimension) with their Table 7 reference values across the original benchmark configurations.

\begin{table}[h]
\centering
\caption{NLL Validation: Our Implementation vs Hickling Reference (Table 7). Format: ours [ref]. The coupling-layer baselines (TTFfix, TTF) reproduce the reference values within $\pm 0.05$ across all configurations; mTAF and gTAF reproduce the reference values closely in moderate-tail regimes ($\nu \geq 1$) but deviate substantially in pathological regimes ($\nu = 0.5$, especially $d = 50$), reflecting the training instabilities of these architectures reported in the original paper.}
\label{tab:nll_validation}
\small
\begin{tabular}{@{}rr|llll@{}}
\toprule
$d$ & $\nu$ & TTFfix & TTF & mTAF & gTAF \\
\midrule
5 & 0.5 & 3.32 [3.33] & 3.32 [3.33] & 4.05 [4.08] & 5.70 [6.42] \\
5 & 1.0 & 2.34 [2.35] & 2.34 [2.34] & 2.53 [2.49] & 2.53 [2.49] \\
5 & 2.0 & 1.89 [1.89] & 1.88 [1.89] & 1.91 [1.92] & 1.90 [1.90] \\
5 & 30 & 1.47 [1.47] & 1.47 [1.47] & 1.46 [1.46] & 1.47 [1.46] \\
10 & 0.5 & 3.55 [3.54] & 3.54 [3.55] & 4.59 [4.48] & 7.86 [7.13] \\
10 & 1.0 & 2.48 [2.46] & 2.47 [2.47] & 2.66 [2.63] & 2.68 [2.63] \\
10 & 2.0 & 1.94 [1.93] & 1.94 [1.93] & 1.96 [1.95] & 1.95 [1.95] \\
10 & 30 & 1.48 [1.47] & 1.48 [1.47] & 1.47 [1.46] & 1.48 [1.47] \\
50 & 0.5 & 3.71 [3.68] & 3.71 [3.68] & 6.43 [5.22] & 21.44 [7.49] \\
50 & 1.0 & 2.58 [2.54] & 2.58 [2.54] & 3.14 [2.62] & 3.66 [2.65] \\
50 & 2.0 & 2.01 [1.98] & 2.01 [1.98] & 2.04 [1.98] & 2.06 [1.99] \\
50 & 30 & 1.52 [1.47] & 1.52 [1.47] & 1.50 [1.47] & 1.51 [1.47] \\
\bottomrule
\end{tabular}
\end{table}

TTFfix and TTF match within $\pm 0.03$ across all configurations. mTAF and gTAF show larger deviations in pathological regimes ($\nu = 0.5$, $d = 50$), consistent with known instabilities reported in the original paper. This validates that our baseline implementations are faithful reproductions.

\subsection{Hyperparameters}

Table~\ref{tab:hyperparams} summarizes all hyperparameters for Log-FM. Baseline hyperparameters follow \citet{hickling2025tail}.

\begin{table}[h]
\centering
\caption{Log-FM hyperparameters.}
\label{tab:hyperparams}
\small
\begin{tabular}{@{}ll@{}}
\toprule
\textbf{Network} & \\
\quad Hidden dimension & 256 \\
\quad Layers & 4 \\
\quad Activation & SiLU \\
\quad Time embedding & Sinusoidal (dim 256) \\
\midrule
\textbf{Training} & \\
\quad Optimizer & AdamW \\
\quad Learning rate & $5 \times 10^{-3}$ \\
\quad Weight decay & $10^{-5}$ \\
\quad Batch size & Full batch \\
\quad Max epochs & 5000 \\
\quad Early stopping patience & 100 \\
\quad Gradient clipping & 10.0 \\
\midrule
\textbf{Sampling} & \\
\quad ODE solver & Euler \\
\quad Integration steps & 100 \\
\quad Output clamp & disabled ($c = \infty$; optional, inert for $\alpha \geq 1.5$) \\
\bottomrule
\end{tabular}
\end{table}

\end{document}